\documentclass[twoside,11pt]{article}

\usepackage{jmlr2e}
\usepackage{amsfonts,epsfig,color,float,amsmath}
\RequirePackage{algorithm}
\RequirePackage{algorithmic}

\ShortHeadings{Importance Weighted Active Learning}{Beygelzimer, Dasgupta and Langford}
\firstpageno{1}

\title{Importance Weighted Active Learning}

\author{\name Alina Beygelzimer \email beygel@us.ibm.com \\
       \addr IBM Thomas J. Watson Research Center \\
       Hawthorne, NY 10532, USA
       \AND
       \name Sanjoy Dasgupta \email dasgupta@cs.ucsd.edu \\
       \addr University of California, San Diego\\
       La Jolla, CA 92093, USA
       \AND
       \name John Langford \email jl@yahoo-inc.com \\
       \addr Yahoo! Research\\
       New York, NY 10018, USA}

\editor{}

\newcommand{\ignore}[1]{}

\def\E{\mathbf{E}}
\def\P{\mathbf{P}}
\def\R{\mathbb{R}}

\newcommand{\rt}{\operatorname{rejection-threshold}}

\begin{document} 
\maketitle

\begin{abstract}
We present a practical and statistically consistent scheme for actively 
learning binary classifiers under general loss functions. Our algorithm
uses importance weighting to correct sampling bias, and by
controlling the variance, we are able to give rigorous 
label complexity bounds for the learning process.
Experiments on passively labeled data show that this approach 
reduces the label complexity required to achieve 
good predictive performance on many learning problems.
\end{abstract}


\section{Introduction}
Active learning is typically defined by contrast to the
passive model of supervised learning. In passive learning, all the 
labels for an unlabeled dataset are obtained at once, while in active 
learning the learner interactively chooses which data points 
to label. The great hope of active learning is that interaction can
substantially reduce the number of labels required, making learning
more practical.  This hope is known to be
valid in certain special cases, where the number of labels needed
to learn actively has
been shown to be logarithmic in the usual sample complexity of
passive learning;
such cases include thresholds on a line, and linear separators with a
spherically uniform unlabeled data distribution~\citep{Linear}.

Many earlier active learning algorithms, such as~\citep{CAL,Linear}, 
have problems with data that are not perfectly separable under the given 
hypothesis class.  In such cases, they 
can exhibit a lack of statistical consistency: even with an infinite
labeling budget, they might not converge to an optimal predictor 
(see \citet{ALtree} for a discussion).

This problem has recently been addressed in two threads of research.
One approach~\citep{A^2,DHM,Hanneke} constructs learning algorithms
that explicitly use sample complexity bounds to assess which hypotheses
are still ``in the running'' (given the labels seen so far), 
thereby assessing the relative value of different unlabeled points 
(in terms of whether they help distinguish between the remaining
hypotheses). These algorithms have the usual PAC-style convergence 
guarantees, but they also have rigorous label complexity bounds that 
are in many cases significantly better than the bounds for passive 
supervised learning. 
However, these algorithms have yet to see practical
use. {First}, they are built explicitly for
$0$--$1$ loss and are not easily adapted to most other loss functions.
This is problematic because in many applications, other loss functions
are more appropriate for describing the problem, or make learning more
tractable (as with convex proxy losses on linear representations).
{Second}, these algorithms make internal use of generalization 
bounds that are often loose in practice, and they can thus end up
requiring far more labels than are really necessary.
{Finally}, they typically require an explicit enumeration over 
the hypothesis class (or an $\epsilon$-cover thereof), 
which is generally computationally intractable.

The second approach to active learning uses 
importance weights to correct sampling bias~\citep{Bach,Sugiyama}.
This approach has only been analyzed in limited
settings. For example,~\citep{Bach} considers linear models and provides 
an analysis of consistency in cases where either (i) the model class
fits the data perfectly, or (ii) the sampling strategy is
non-adaptive (that is, the data point queried at time $t$ doesn't
depend on the sequence of previous queries).
The analysis in these works is also asymptotic rather than yielding
finite label bounds, while minimizing the actual label complexity is
of paramount importance in active learning.
Furthermore, the analysis does not prescribe how to choose 
importance weights, and a poor choice can result in high label complexity.


\subsection*{Importance-weighted active learning}  
We address the problems above with an active learning
scheme that provably yields PAC-style label complexity
guarantees. When presented with an unlabeled point $x_t$, this scheme
queries its label with a carefully chosen probability $p_t$, taking
into account the identity of the point and the history of labels seen
so far. The points that end up getting labeled are then weighted
according to the reciprocals of these probabilities (that is,
$1/p_t$), in order to remove sampling bias.  We show (theorem~\ref{thm:safety})
 that this simple
method guarantees statistical consistency: for any distribution and
any hypothesis class, active learning eventually converges to the
optimal hypothesis in the class.

As in any importance sampling scenario, the biggest challenge is controlling
the variance of the process. This depends crucially on how the sampling
probability $p_t$ is chosen. Our strategy, roughly, is to make it proportional 
to the spread of values $h(x_t)$, as $h$ ranges over the remaining candidate 
hypotheses (those with good performance on the labeled points 
so far). For this setting of $p_t$, which we call IWAL({loss-weighting}),
we have two results. 
First, we show (theorem~\ref{loss-bound}) 
a fallback guarantee that the label complexity is 
never much worse than that of supervised learning. Second, we rigorously 
analyze the label complexity in terms of underlying parameters of the 
learning problem (theorem~\ref{thm:label}). 
Previously, label complexity bounds for active learning 
were only known for $0$--$1$ loss, and were based on the 
{\it disagreement coefficient} of the learning problem~\citep{Hanneke}. 
We generalize this notion to general loss functions, and analyze label 
complexity in terms of it. We consider settings in which these 
bounds turn out to be roughly the {\it square root} of the sample complexity 
of supervised learning.

In addition to these upper bounds, we 
show a general lower bound on the label complexity of active learning 
(theorem~\ref{thm:lower-bound})
that significantly improves the best previous such result~\citep{Lower-Bound}.

We conduct practical experiments with two IWAL algorithms.  The first 
is a specialization of IWAL({loss-weighting}) to the case of 
linear classifiers with convex loss functions; here, the 
algorithm becomes tractable via convex programming 
(section~\ref{sec:easy-alg}).  The second, IWAL(bootstrap), 
uses a simple bootstrapping scheme that reduces active learning 
to (batch) passive learning without requiring 
much additional computation (section~\ref{sec:bootstrap}).  
In every case, these experiments yield 
substantial reductions in label complexity compared to passive learning, 
without compromising predictive performance. They suggest 
that IWAL is a practical scheme that can reduce the label complexity 
of active learning without sacrificing the statistical guarantees 
(like consistency) we take for granted in passive learning. 

\subsection*{Other related work}
The active learning algorithms of \citet{Abe}, based on boosting and 
bagging, are similar in spirit 
to our IWAL(bootstrap) algorithm in section~\ref{sec:bootstrap}.  
But these earlier algorithms are not consistent in the presence of 
adversarial noise: they may never converge to the correct solution, 
even given an infinite label budget.  In contrast, IWAL(bootstrap) 
is consistent and satisfies further guarantees (section~\ref{sec:def}).

The field of {\it experimental design}
\citep{P06} emphasizes regression problems in
which the conditional distribution of the response variable given the
predictor variables is assumed to lie in a certain class; the goal
is to synthesize query points such that the resulting least-squares
estimator has low variance. 
In contrast, we are
interested in an agnostic setting, where no assumptions about the
model class being powerful enough to represent the ideal solution
exist. Moreover, we are not allowed to synthesize queries, but merely
to choose them from a stream (or pool) of candidate queries provided
to us. A telling difference between the two models is
that in experimental design, it is common to query the same point
repeatedly, whereas in our setting this would make no sense.

\section{Preliminaries}
\label{sec:def}
Let $X$ be the input space and $Y$ the output space.
We consider active learning in the streaming setting where at each step $t$,
a learner observes an unlabeled point $x_t\in X$ and has to decide
whether to ask for the label $y_t\in Y$.  The learner works 
with a hypothesis space $H=\{h: X\rightarrow Z\}$, where $Z$ is 
a prediction space.

The algorithm is evaluated with respect to a given loss function
$l: Z \times Y \rightarrow[0,\infty)$. The most common loss function is
$0$--$1$ loss, in which  $Y = Z = \{-1,1\}$ and 
$l(z,y)={\bf 1}(y\not=z) = {\bf 1}(yz < 0)$.  
The following examples address the binary case $Y=\{-1,1\}$ with $Z \subset \R$:
\begin{itemize}
\item
$l(z,y)=(1-yz)_{+}$ (hinge loss), 
\item
$l(z,y)=\ln(1+e^{-yz})$ (logistic loss), 
\item
$l(z,y)=(y-z)^{2}=(1-yz)^{2}$ (squared loss), and
\item
$l(z,y)=|y-z|=|1-yz|$ (absolute loss). 
\end{itemize}
Notice that all the loss functions mentioned here are of the form
$l(z,y) = \phi(yz)$ for some function $\phi$ on the reals. We
specifically highlight this subclass of loss functions when proving
label complexity bounds.
Since these functions are bounded (if $Z$ is), we further assume
they are normalized to output a value in $[0,1]$.

\section{The Importance Weighting Skeleton}
\label{sec:skel}
Algorithm~\ref{alg:IWAL} describes the basic outline of importance-weighted 
active learning (IWAL). 
Upon seeing $x_t$, the learner 
calls a subroutine {\it rejection-threshold} 
(instantiated in later sections), which looks 
at $x_t$ and past history
to return the probability $p_t$ of requesting $y_t$.

The algorithm maintains a set of labeled examples seen so far, each with
an importance weight:
if $y_t$ ends up being queried, 
its weight is set to $1/p_t$.

\begin{algorithm}[h]
\caption{\label{alg:IWAL}IWAL (subroutine $\rt$)}
Set $S_{0}=\emptyset$. \\
For $t$ from $1,2,\ldots$ until the data stream runs out:
\begin{enumerate}
\item Receive $x_{t}$ .
\item Set $p_{t}=\rt(x_{t},\{x_i,y_i,p_i,Q_i : 1\leq i < t\})$.
\item Flip a coin $Q_t\in\{0,1\}$ with $\E[Q_t]=p_t$. \\If $Q_t=1$, request
$y_{t}$ and set \mbox{$S_{t}=S_{t-1}\cup\{(x_{t},y_{t},1/p_{t})\}$},
else $S_t=S_{t-1}$.
\item Let $h_{t}=\arg\min_{h\in H}\sum_{(x,y,c)\in S_{t}} c \cdot l(h(x),y)$.
\end{enumerate}
\end{algorithm}

Let $D$ be the underlying probability distribution on $X\times Y$.
The expected loss of $h\in H$ on $D$ is given by
$L(h)=\E_{(x,y)\sim D}\, l(h(x),y)$. 
Since $D$ is always clear from context, we drop it from notation.
The importance weighted estimate of the loss at time $T$ is
$$
L_{T}(h) = \frac{1}{T}\sum_{t=1}^{T}\frac{Q_{t}}{p_{t}}\, l(h(x_{t}),y_{t}),
$$
where $Q_t$ is as defined in the algorithm.
It is easy to see that $\E[ L_{T}(h)] = L(h)$, with the expectation
taken over all the random variables involved. 
Theorem~\ref{loss-bound} gives large deviation bounds for $L_{T}(h)$,
{provided} that the probabilities $p_{t}$ are chosen carefully.

\subsection{A safety guarantee for IWAL}
A desirable property for a learning algorithm is \emph{consistency}:
Given an infinite budget of unlabeled and labeled examples, does it 
converge to the best predictor?  Some early
active learning algorithms~\citep{CAL,Linear} do not satisfy this
baseline guarantee: they have problems if the data cannot be classified
perfectly by the given hypothesis class. We prove that
IWAL algorithms are consistent, as long as $p_{t}$ is bounded away
from $0$.  Further, we prove that the label complexity required is
within a constant factor of supervised learning in the worst case.

\begin{thm}\label{thm:safety}
For all distributions $D$, for all finite hypothesis classes $H$, for 
any $\delta > 0$, if there is a constant $p_{\min}>0$ 
such that $p_{t} \geq p_{\min}$ for all $1\leq t\leq T$, then
\[
\P\left[ \max_{h \in H} |L_{T}(h) - L(h)| > \frac{\sqrt{2}}{p_{\mbox{\rm {\scriptsize min}}}}\sqrt{\frac{\ln |H| + \ln \frac{2}{\delta}}{T}} \right] < \delta.
\]
\end{thm}
Comparing this result to the usual sample complexity bounds 
in supervised learning (for example, corollary 4.2 of~\citep{tutorial}), 
we see that the label complexity is at most 
$2/p_{\mbox{{\scriptsize min}}}^2$ times that of a supervised algorithm.
For simplicity, the bound is given in terms of $\ln |H|$ rather 
than the VC dimension of $H$. 
The argument, which is a martingale modification of standard results,
can be extended to VC spaces.
\bigskip\noindent
\begin{proof} 
Fix the underlying distribution. 
For a hypothesis $h\in H$,
consider a sequence of random variables $U_{1},\ldots,U_{T}$ with
$$
U_t \ =\ \frac{Q_t}{p_{t}}l(h(x_{t}),y_{t})-L(h).$$
Since $p_{t}\geq p_{\mbox{\scriptsize min}}$, $|U_{t}|\leq 1/p_{\mbox{\scriptsize min}}$.
The sequence $Z_{t}=\sum_{i=1}^{t}U_{i}$ is a martingale, letting
$Z_0=0$.  Indeed, for any $1\leq t\leq T$,
\begin{align*}
\E[Z_{t} \mid Z_{t-1},\ldots,Z_{0}] &= \E_{Q_{t},x_{t},y_{t},p_t}\left[U_t + Z_{t-1} \mid Z_{t-1},\ldots,Z_{0}\right]\\
& = Z_{t-1} + \E_{Q_t, x_{t},y_{t},p_t}\left[\frac{Q_t}{p_t} l(h(x_{t}),y_{t}) - L(h) \bigg| Z_{t-1},\ldots,Z_{0}\right]\\
&=\E_{x_{t},y_{t}}\left[\, l(h(x_{t}),y_{t})-L(h) + Z_{t-1} \mid Z_{t-1},\ldots,Z_{0}\,\right]=Z_{t-1}.\end{align*}
Observe that $|Z_{t+1}-Z_{t}|=|U_{t+1}|\leq1/p_{\mbox{\scriptsize min}}$ for all $0\leq t<T$.
Using $Z_T = T(L_T(h)-L(h))$ and applying Azuma's inequality \citep{A67}, we see that for any 
$\lambda > 0$,
$$
\P\left[ |L_{T}(h)-L(h)|>\frac{\lambda}{p_{\min}\sqrt{T}} \right]
=\P\left[Z_{T} > \frac{\lambda\sqrt{T}}{p_{\min}} \right] 
< 2e^{-\lambda^2/2}.
$$
Setting $\lambda = \sqrt{2(\ln|H| + \ln (2/\delta))}$ and taking a union bound over
$h \in H$ then yields the desired result.
\end{proof}

\section{Setting the Rejection Threshold: Loss Weighting}
\label{sec:IWAL}
Algorithm~\ref{alg:threshold} gives a particular instantiation of the rejection 
threshold subroutine in IWAL. The subroutine 
maintains an effective hypothesis class $H_t$, 
which is initially all of $H$ and then gradually shrinks by 
setting $H_{t+1}$ to the subset of $H_t$ whose 
empirical loss isn't too much worse than $L_t^*$, the smallest 
empirical loss in $H_t$:
$$ H_{t+1} = \{h \in H_t: L_t(h) \leq L_t^* + \Delta_t \}.$$
The allowed slack
$\Delta_t=\sqrt{(8/t)\ln (2t(t+1)|H|^{2}/\delta)}$
comes from a standard sample complexity bound. 

We will show that, with high probability, any optimal hypothesis $h^*$ is always
in $H_t$, and thus
all other hypotheses can be discarded from consideration.
For each $x_t$,
the loss-weighting scheme looks at the range of
predictions on $x_t$ made by hypotheses in $H_t$ and sets 
the sampling probability
$p_t$ to the size of this range. More precisely, 
$$ p_t = \max_{f,g \in H_t} \max_y l(f(x_t),y) - l(g(x_t),y).$$
Since the loss values are normalized to lie in $[0,1]$, we
can be sure that $p_t$ is also in this interval.
Next section shows that the resulting IWAL has several desirable properties.


\begin{algorithm}[h]
\caption{\label{alg:threshold} loss-weighting ($x$, $\{x_i,y_i,p_i,Q_i: i < t\}$)}
\begin{enumerate}
\item Initialize $H_0=H$.
\item Update
\vskip -.3in
\begin{align*}
L_{t-1}^{*} & = \min_{h\in H_{t-1}} \frac{1}{t-1} \sum_{i=1}^{t-1}\frac{Q_i}{p_i}l(h(x_{i}),y_{i}), \\
H_t & = \left\{h\in H_{t-1}:
\frac{1}{t-1}\sum_{i=1}^{t-1}\frac{Q_i}{p_i}l(h(x_{i}),y_{i})
\leq L_{t-1}^{*}+\Delta_{t-1} \right\}. 
\end{align*}
\item 
Return $p_t = \max_{f,g\in H_{t},y\in Y}l(f(x),y)-l(g(x),y)$.
\end{enumerate}
\end{algorithm}

\subsection{A generalization bound}\label{loss-weighting-bound}
We start with a large deviation bound for each $h_t$ output by 
IWAL(loss-weighting).  It is not a corollary of 
theorem~\ref{thm:safety} because it does not require the sampling
probabilities be bounded below away from zero.

\begin{thm}\label{loss-bound} Pick any data distribution $D$ and 
hypothesis class $H$, and let $h^* \in H$ be a minimizer of the loss
function with respect to $D$.  Pick any $\delta > 0$.
With probability 
at least $1-\delta$, for any $T\geq 1$,
\begin{itemize}
\item[$\circ$]
$h^*\in H_T$, and
\item[$\circ$]
$L(f) - L(g) \leq 2 \Delta_{T-1}$ for any $f,g \in H_T$.
\end{itemize}
In particular, if $h_T$ is the output of \emph{IWAL(loss-weighting)},
then $L(h_T) - L(h^*) \leq 2\Delta_{T-1}$.
\end{thm}

\bigskip\noindent
We need the following lemma for the proof.
\begin{lemma}\label{lem:pair} For all data distributions $D$, for
all hypothesis classes $H$, for all $\delta > 0$, with probability
at least $1-\delta$, for all $T$ and all $f,g\in H_{T}$, 
$$|L_T(f)-L_T(g) - L(f) + L(g)| \leq \Delta_T.$$
\end{lemma}
\begin{proof}
Pick any $T$ and $f,g\in H_T$.
Define 
$$Z_{t}=\frac{Q_{t}}{p_{t}}\big(l(f(x_{t}),y_{t})-l(g(x_{t}),y_{t})\big)-(L(f)-L(g)).$$
Then 
$
\ \E\left[Z_{t}\ |\ Z_{1},\ldots,Z_{t-1}\right]
 = \E_{x_{t},y_{t}}\left[\,l(f(x_{t}),y_{t})-l(g(x_{t}),y_{t})-(L(f)-L(g))\ |\ Z_{1},\ldots,Z_{t-1}\right] = 0$.
Thus $Z_{1},Z_{2},\ldots$ is a martingale difference sequence, and 
we can use Azuma's inequality to show that its sum is tightly 
concentrated, if the individual $Z_t$ are bounded.

To check boundedness, observe that since $f$ and $g$ are in $H_T$, they must
also be in $H_{1},H_{2},\ldots,H_{T-1}$. Thus for all $t\leq T$,
$ p_{t}\geq |l(f(x_{t}),y_{t})-l(g(x_{t}),y_{t})|$,
whereupon
$
|Z_t| \leq \frac{1}{p_{t}}|l(f(x_{t}),y_{t})-l(g(x_{t}),y_{t})| + 
|L(f)-L(g)| \leq 2.
$

We allow failure probability $\delta/T(T+1)$ at time $T$.
Applying Azuma's inequality, we have
\begin{eqnarray*}
\lefteqn{\P[|L_{T}(f)-L_{T}(g) - L(f)+L(g)|\geq\Delta_{T}]}\\
& = & \P\left[ \left| \frac{1}{T}\left(\sum_{t=1}^{T}\left(\frac{Q_{t}}{p_{t}}(l(f(X_{t}),Y_{t})-l(g(X_{t}),Y_{t}))-(L(f)-L(g))\right)\right) \right| \geq\Delta_{T}\right]\\
 & = & \P\left[ \left| \sum_{t=1}^{T}Z_{t} \right| \geq T\Delta_{T}\right]
 \leq 2 e^{-T\Delta_{T}^{2}/8}
= \frac{\delta}{T(T+1)|H|^{2}}.\end{eqnarray*}
\noindent 
Since $H_T$ is a random subset of $H$, it suffices to take a union bound over all $f,g\in H$, and $T$.  A union bound over $T$ finishes the
proof.
\end{proof}
\begin{proof} 
(Theorem~\ref{loss-bound})\quad
Start by assuming that the $1-\delta$ probability event of
lemma~\ref{lem:pair} holds.
We first show by induction that $h^* = \arg \min_{h\in H} L(h)$
is in $H_T$ for all $T$. It holds at $T=1$, 
since $H_1 = H_0 = H$.  Now suppose it holds at $T$, and show 
that it is true at $T+1$.  
Let $h_{T}$ minimize $L_{T}$ over $H_T$.  
By lemma~\ref{lem:pair},
$
L_{T}(h^{*})-L_{T}(h_{T})
\leq L(h^{*})-L(h_{T})+\Delta_{T}
\leq \Delta_{T}.
$
Thus $L_{T}(h^{*})\leq L_{T}^{*}+\Delta_{T}$ and hence $h^{*}\in H_{T+1}$.

Since $H_T \subseteq H_{T-1}$, lemma~\ref{lem:pair} implies that for 
for any $f,g\in H_T$,
$$
L(f)-L(g) \leq L_{T-1}(f)-L_{T-1}(g)+\Delta_{T-1}
\leq L_{T-1}^{*}+\Delta_{T-1}-L_{T-1}^{*}+\Delta_{T-1}
 = 2\Delta_{T-1}.
$$
Since $h_T, h^* \in H_T$, we have $L(h_T) \leq L(h^*) + 2\Delta_{T-1}$.
\end{proof}

\section{Label Complexity}
\label{sec:label}
We showed that the loss of the classifier output by IWAL(loss-weighting) 
is similar to the loss
of the classifier chosen passively after seeing all $T$ labels.
How many of those $T$ labels does the active 
learner request? 

\citet{DHM} studied this question for an active learning 
scheme under $0$--$1$ loss.  For learning problems with 
bounded {\it disagreement coefficient}~\citep{Hanneke}, 
the number of queries was found to be
$O(\eta T + d \log^2 T)$, where $d$ is the VC dimension of the 
function class, and $\eta$ is the best error rate achievable on the underlying
distribution by that function class.
We will soon see (section~\ref{sec:lower-bound}) that the term $\eta T$ 
is inevitable for any active learning scheme; 
the remaining term has just a polylogarithmic dependence on $T$.

We generalize the disagreement coefficient to arbitrary loss functions and
show that, under conditions similar to the earlier result, the number of 
queries is
$O \left(\eta T + \sqrt{dT \log^2 T} \right)$, where $\eta$ is now
the best achievable loss. The inevitable $\eta
T$ is still there, and the second term is still sublinear, though 
not polylogarithmic as before. 

\subsection{Label Complexity: Main Issues}
Suppose the loss function is minimized by $h^* \in H$, with
$L^* = L(h^*)$. Theorem~\ref{loss-bound} shows that at
time $t$, the remaining hypotheses $H_t$ include $h^*$ and 
all have losses in the range $[L^*, L^* + 2 \Delta_{t-1}]$. We
now prove that under suitable conditions, the sampling
probability $p_t$ has expected value $\approx L^* + \Delta_{t-1}$.
Thus the expected total number of labels queried upto time $T$ is
roughly 
$L^* T + \sum_{t=1}^T \Delta_{t-1} \approx L^* T + \sqrt{T \ln |H|}$.

To motivate the proof, consider a loss function $l(z,y) =
\phi(yz)$; all our examples are of this form. 
Say $\phi$ is differentiable with $0 < C_0
\leq |\phi'| \leq C_1$. Then the sampling probability for $x_t$ is
\ignore{$p_t 
= \max\limits_{f,g \in H_t} \max\limits_{y\in\{-1,+1\}} \ l(f(x_t),y) - l(g(x_t),y) 
= \max\limits_{f,g \in H_t} \max\limits_y \ \phi(yf(x_t)) - \phi(yg(x_t)) 
\leq C_1 \max\limits_{f,g \in H_t} \max\limits_y \ |yf(x_t) - yg(x_t)| 
= C_1 \max\limits_{f,g \in H_t} \ |f(x_t) - g(x_t)| 
\leq 2C_1 \max\limits_{h \in H_t} \ | h(x_t) - h^*(x_t) |$.}
\begin{eqnarray*}
p_t 
&  = & \max_{f,g \in H_t} \max_{y\in\{-1,+1\}} \ l(f(x_t),y) - l(g(x_t),y)  \\
& = & \max_{f,g \in H_t} \max_y \ \phi(yf(x_t)) - \phi(yg(x_t))  \\
& \leq & C_1 \max_{f,g \in H_t} \max_y \ |yf(x_t) - yg(x_t)|  \\
& = & C_1 \max_{f,g \in H_t} \ |f(x_t) - g(x_t)|  \\
& \leq & 2C_1 \max_{h \in H_t} \ | h(x_t) - h^*(x_t) | .
\end{eqnarray*}
So $p_t$ is determined by the range of predictions on $x_t$ by hypotheses
in $H_t$. Can we bound the size of this range, given that any $h \in H_t$
has loss at most $L^* + 2 \Delta_{t-1}$? 
\begin{eqnarray*}
2 \Delta_{t-1} & \geq & L(h) - L^* \\
& \geq & \E_{x,y} |l(h(x), y) - l(h^*(x),y)| - 2 L^* \\
& \geq & \E_{x,y} C_0 |y(h(x) - h^*(x))| - 2L^* \\
& = &  C_0 \E_x |h(x) - h^*(x)| - 2L^*.
\end{eqnarray*}
So we can upperbound $\max_{h \in H_t} \E_x |h(x) - h^*(x)|$
(in terms of $L^*$ and $\Delta_{t-1}$), whereas we want to upperbound
the expected value of $p_t$, which is proportional to 
$\E_x \max_{h \in H_t} |h(x) - h^*(x)|$. The ratio between these two 
quantities is related to a fundamental parameter of the learning problem,
a generalization of the {\it disagreement coefficient} 
\citep{Hanneke}.

We flesh out this intuition in the remainder of this section. First 
we describe a broader class of loss functions than those
considered above (including $0$--$1$ loss, which is not differentiable);
a distance metric on hypotheses, and a generalized disagreement coefficient.  
We then prove that 
for this broader class, active learning performs better than passive 
learning when the generalized disagreement coefficient is small.  

\subsection{A subclass of loss functions}

We give label complexity upper bounds for a class of loss
functions that includes $0$--$1$ loss and logistic loss but not hinge
loss. Specifically, we require that the loss function has bounded 
{\it slope asymmetry}, defined below.

Recall earlier notation: response space $Z$, classifier space $H = \{h: X
\rightarrow Z\}$,  and loss
function $l: Z \times Y \rightarrow [0,\infty)$. Henceforth, 
the label space is $Y = \{-1,+1\}$.

\begin{defn}
The {\em slope asymmetry} of a loss function $l: Z \times Y \rightarrow [0,\infty)$ 
is
$$ K_l = \sup_{z,z' \in Z} 
\frac{\max_{y \in Y} \left| l(z,y) - l(z',y) \right|}{\min_{y \in Y} \left| l(z,y) - l(z',y) \right|} .$$
\end{defn}
The slope asymmetry is $1$ for $0$--$1$ loss,
and $\infty$ for hinge loss. For differentiable loss functions
$l(z,y) = \phi(yz)$, it is easily related to bounds on the derivative.
\begin{lemma}
Let $l_\phi(z,y)=\phi(zy)$, where $\phi$ is a differentiable function 
defined on $Z = [-B,B] \subset \R$. Suppose $C_0 \leq |\phi'(z)| \leq C_1$ 
for all $z \in Z$. Then for any $z,z' \in Z$, and any $y \in \{-1,+1\}$,
$$ C_0 |z - z'| \leq \ | l_\phi(z,y) - l_\phi(z',y)| \leq C_1 |z - z'|.$$
Thus $l_\phi$ has slope asymmetry at most $C_1/C_0$.
\label{lemma:convex-bound}
\end{lemma}
\begin{proof}
By the mean value theorem, there is some $\xi \in Z$ such that
$ l_\phi(z,y) - l_\phi(z',y) = \phi(yz) - \phi(yz')
=  \phi'(\xi)(yz - yz')$.
Thus
$ | l_\phi(z,y) - l_\phi(z',y)| =  |\phi'(\xi)| \cdot |z-z'|$, and
the rest follows from the bounds on $\phi'$.
\end{proof}
For instance, this immediately applies to logistic
loss.
\begin{cor}
Logistic loss $l(z,y) = \ln(1 + e^{-yz})$, defined on label
space $Y = \{-1,+1\}$ and response space $[-B,B]$, has slope asymmetry
at most $1 + e^B$.
\end{cor}

\subsection{Topologizing the space of classifiers}
We introduce a simple distance function on the space of classifiers.

\begin{defn}
For any $f,g \in H$ and distribution $D$ define $\rho(f,g) = \E_{x\sim
  D} \max_y |l(f(x),y) - l(g(x),y)|$.  For any $r \geq 0$, let $B(f,r)
= \{g \in H: \rho(f,g) \leq r\}$.
\end{defn}
Suppose $L^* = \min_{h \in H} L(h)$ is realized at $h^*$. 
We know that at time $t$, the remaining hypotheses have loss at most 
$L^* + 2\Delta_{t-1}$. Does this mean they are close to $h^*$
in $\rho$-distance? The ratio between the two can
be expressed in terms of the slope asymmetry of the loss.

\begin{lemma}
For any distribution $D$ and any loss function with slope asymmetry
$K_l$, we have $\rho(h,h^*) \leq K_l(L(h) + L^*)$ for all $h \in H$.
\label{lemma:two-spaces}
\end{lemma}
\begin{proof}
For any $h \in H$, 
\begin{align*}
\rho(h,h^*) \
& = \
\E_{x} \textstyle \max_y |l(h(x),y) - l(h^*(x),y)| \\
& \leq \
K_l \, \E_{x,y} |l(h(x),y) - l(h^*(x),y)| \\
& \leq \
K_l \, \left( \E_{x,y} [l(h(x),y)] + \E_{x,y} [l(h^*(x),y)] \right) \\
& = \
K_l  \, ( L(h) + L(h^*) ) .
\end{align*}
\end{proof}

\subsection{A generalized disagreement coefficient}
When analyzing the $A^2$ algorithm~\citep{A^2} for active learning
under $0$--$1$ loss, \citet{Hanneke} found that its label
complexity could be characterized in terms of what he called the {\it
  disagreement coefficient} of the learning problem. We now generalize
this notion to arbitrary loss functions.

\begin{defn}
The {\it disagreement coefficient} is the infimum value of $\theta$ 
such that for all $r$,
$$ \E_{x \sim D} \textstyle \sup_{h \in B(h^*, r)} \sup_y |l(h(x),y) - l(h^*(x),y)| \, \leq \, \theta r.$$
\end{defn}
Here is a simple example for linear separators.
\begin{lemma}
Suppose $H$ consists of linear classifiers $\{u \in \R^d: \|u\| \leq B\}$
and the data distribution $D$ is uniform over the surface of the 
unit sphere in $\R^d$. Suppose the loss function is $l(z,y) = \phi(yz)$
for differentiable $\phi$ with $C_0 \leq |\phi'| \leq C_1$. Then the disagreement 
coefficient is at most $(2C_1/C_0) \sqrt{d}$.
\end{lemma}

\begin{proof}
Let $h^*$ be the optimal classifier, and $h$ any other classifier with 
$\rho(h,h^*) \leq r$.
Let $u^*, u$ be the corresponding vectors in $\R^d$. 
Using lemma~\ref{lemma:convex-bound},
\begin{align*}
r & \geq \E_{x \sim D} \sup_y |l(h(x), y) - l(h^*(x), y)|  \\
& \geq C_0 \, \E_{x \sim D} |h(x) - h^*(x)| \\
&  =  C_0 \, \E_{x \sim D} |(u-u^*) \cdot x | 
\ \geq \ C_0  \, \|u-u^*\|/(2\sqrt{d}) .
\end{align*}
Thus for any $h \in B(h^*,r)$, we have that the corresponding vectors satisfy
$\|u-u^*\| \leq 2 r \sqrt{d}/C_0$.
We can now bound the disagreement coefficient:
\begin{eqnarray*}
\lefteqn{\E_{x \sim D} \sup_{h \in B(h^*,r)} \sup_y |l(h(x),y) - l(h^*(x),y)| }\\
& \leq &
C_1 \, \E_{x \sim D} \sup_{h \in B(h^*,r)} |h(x) - h^*(x)| \\
& \leq &
C_1 \, \E_{x} \sup\{|(u-u^*) \cdot x|: \|u-u^*\| \leq 2r\sqrt{d}/C_0\} \\
& \leq & C_1 \, \cdot 2r \sqrt{d}/C_0 .
\end{eqnarray*}
\end{proof}
\subsection{Upper Bound on Label Complexity}
Finally, we give a bound on label complexity for learning problems
with bounded disagreement coefficient and loss functions with
bounded slope asymmetry. 

\begin{thm}\label{thm:label} 
For all learning problems $D$ and hypothesis spaces $H$, 
if the loss function has slope asymmetry $K_l$, and the learning
problem has disagreement coefficient $\theta$, then for all $\delta > 0$,
with probability at least $1-\delta$ over the choice of data,
the expected number of labels requested by \emph{IWAL(loss-weighting)} during 
the first $T$ iterations is at most
$$4\theta \cdot K_l \cdot 
(L^* T + O( \sqrt{T \ln (|H| T/\delta)} ) ),$$
where $L^*$ is the minimum loss achievable on $D$ by $H$, and the expectation
is over the randomness in the selective sampling.
\end{thm}
\begin{proof}
Suppose $h^*\in H$ achieves loss $L^*$. 
Pick any time $t$. By theorem~\ref{loss-bound}, 
$H_t \subset \{h \in H: L(h) \leq L^* + 2 \Delta_{t-1} \}$ and by 
lemma~\ref{lemma:two-spaces}, $H_t \subset B(h^*, r)$ for 
$r = K_l (2L^* + 2 \Delta_{t-1})$. Thus, the expected value of 
$p_t$ (over the choice of $x$ at time $t$) is at most
\begin{align*}
\E_{x \sim D} \sup_{f,g \in H_t} \sup_y |l(f(x),y) - l(g(x),y)| 
& \leq 2 \, \E_{x \sim D} \sup_{h \in H_t} \sup_y |l(h(x),y) - l(h^*(x),y)| \\
& \leq 2 \, \E_{x \sim D} \sup_{h \in B(h^*,r)} \sup_y |l(h(x),y) - l(h^*(x),y)| \\
& \leq  2 \theta r
= 4 \theta \cdot K_l \cdot \left( L^* + \Delta_{t-1} \right) .
\end{align*}
Summing over $t = 1,\ldots,T$, we get the lemma.
\end{proof}

\subsection{Other examples of low label complexity}
It is also sometimes possible to achieve substantial label
complexity reductions over passive learning, even when the slope
asymmetry is infinite.

\begin{example}
Let the space $X$ be the ball of radius $1$ in $d$ dimensions.

Let the distribution $D$ on $X$ be a point mass at the origin with
weight $1 - \beta$ and label $1$ and a point mass at $(1,0,0,\ldots,0)$
with weight $\beta$ and label $-1$ half the time and label $0$ for the
other half the time.

Let the hypothesis space be linear with weight vectors satisfying
$||w|| \leq 1$.

Let the loss of interest be squared loss: $l(h(x),y) = (h(x)-y)^2$
which has infinite slope asymmetry.
\end{example}

\begin{observation}
For the example above, IWAL({\sl loss-weighting}) requires only an expected
$\beta$ fraction of the labeled samples of passive learning to achieve
the same loss.
\end{observation}
\begin{proof}
Passive learning samples from the point mass at the origin a
$(1-\beta)$ fraction of the time, while active learning only samples
from the point mass at $(1,0,0,\ldots,0)$ since all predictors have the
same loss on samples at the origin.

Since all hypothesis $h$ have the same loss for samples at the origin,
only samples not at the origin influence the sample complexity.
Active learning samples from points not at the origin $1/\beta$ more
often than passive learning, implying the theorem.
\end{proof}

\section{A lower bound on label complexity}
\label{sec:lower-bound}

\citep{Lower-Bound} showed that for any
hypothesis class $H$ and any $\eta > \epsilon > 0$, there
is a data distribution such that (a) the 
optimal error rate achievable by $H$ is $\eta$; and (b) any active 
learner that finds $h \in H$ with error rate 
$\leq \eta + \epsilon$ (with probability $> 1/2$) must make
$\eta^2/\epsilon^2$ queries.
We now strengthen this lower bound to $d \eta^2/\epsilon^2$, where $d$ is the 
VC dimension of $H$.

Let's see how this relates to 
the label complexity rates of the previous section. It is
well-known that if a supervised learner sees $T$ examples 
(for any $T > d/\eta$), its final hypothesis has error 
$\leq \eta + \sqrt{d\eta/T}$ \citep{DGL} with high probability. Think of 
this as $\eta + \epsilon$ for 
$\epsilon = \sqrt{d\eta/T}$.  Our lower bound now implies that an active 
learner must make at least $d \eta^2/\epsilon^2 = \eta T$ queries. This
explains the $\eta T$ leading term in all the label complexity bounds
we have discussed.

\begin{thm}
For any $\eta,\epsilon > 0$ such that $2\epsilon \leq \eta \leq 1/4$,
for any input space $X$ and hypothesis class $H$ (of functions
mapping $X$ into $Y = \{+1,-1\}$) of VC dimension $1 < d < \infty$, 
there is a distribution over $X \times Y$ such that (a)
the best error rate achievable by $H$ is $\eta$; (b) any active
learner seeking a classifier of error at most $\eta + \epsilon$
must make $\Omega(d\eta^2/\epsilon^2)$ queries to succeed
with probability at least $1/2$.
\label{thm:lower-bound}
\end{thm}
\begin{proof}
Pick a set of $d$ points $x_o, x_1, x_2, \ldots, x_{d-1}$ shattered by $H$. 
Here is a distribution over $X \times Y$: point $x_o$ has probability 
$1-\beta$, while each of the remaining $x_i$ has 
probability $\beta/(d-1)$, where $\beta = 2(\eta + 2\epsilon)$.
At $x_o$, the response is always $y=1$. At $x_i, i \geq 1$, the response 
is $y=1$ with probability $1/2 + \gamma b_i$, where $b_i$ is either $+1$ 
or $-1$, and 
$\gamma = 2\epsilon/\beta = \epsilon/(\eta + 2\epsilon) < 1/4$.

Nature starts by picking $b_1, \ldots, b_{d-1}$ uniformly at random. This
defines the target hypothesis $h^*$: $h^*(x_o) = 1$ and $h^*(x_i) = b_i$. 
Its error rate is $\beta \cdot (1/2 - \gamma) = \eta$.

Any learner outputs a hypothesis in $H$ and thus implicitly makes guesses
at the underlying hidden bits $b_i$. Unless it correctly determines $b_i$ 
for at least $3/4$ of the points $x_1, \ldots, x_{d-1}$, the error of its 
hypothesis will be at least 
$\eta + (1/4) \cdot \beta \cdot (2 \gamma) = \eta + \epsilon$.

Now, suppose the active learner makes $\leq c (d-1)/\gamma^2$ queries, 
where $c$ is a small constant ($c \leq 1/125$ suffices). We'll show 
that it fails (outputs a hypothesis with error $\geq \eta + \epsilon$) 
with probability at least $1/2$.

We'll say $x_i$ is {\it heavily queried} if the active learner queries
it at least $4c/\gamma^2$ times. At most $1/4$ of the $x_i$'s are heavily queried;
without loss of generality, these are $x_1, \ldots, x_k$, for some 
$k \leq (d-1)/4$. The remaining $x_i$ get so few queries that the learner 
guesses each corresponding bit $b_i$ with probability less than $2/3$; this 
can be derived from Slud's lemma (below), which 
relates the tails of a binomial to that of a normal.

Let $F_i$ denote the event that the learner gets $b_i$ wrong; 
so $\E F_i \geq 1/3$ for $i > k$. Since $k \leq (d-1)/4$,
the probability that the learner fails is given by
\begin{align*}
\P [\mbox{learner fails}]
& =  
\P [F_1 + \cdots + F_{d-1} \geq (d-1)/4] \\
& \geq 
\P [F_{k+1} + \cdots + F_{d-1} \geq (d-1)/4] \\
& \geq 
\P [B \geq (d-1)/4] \geq 
\P [Z \geq 0] = 1/2,
\end{align*}
where $B$ is a $\mbox{binomial}((3/4)(d-1),1/3)$ random variable,
$Z$ is a standard normal, and the last inequality follows from Slud's lemma. 
Thus the active learner must make at least
$c (d-1)/\gamma^2 = \Omega(d \eta^2/\epsilon^2)$ 
queries to succeed with probability at least $1/2$.
\end{proof}

\begin{lemma}[\citet{S77}]
Let $B$ be a Binomial $(n,p)$ random variable with $p \leq 1/2$, and let
$Z$ be a standard normal. For any $k \in [np, n(1-p)]$,
$ \P [B \geq k] \ \geq \ \P [ Z \geq (k-np)/\sqrt{np(1-p)}] .$
\label{lemma:slud}
\end{lemma}
Theorem~\ref{thm:lower-bound} uses the same example that is used for lower 
bounds on supervised sample complexity (section 14.4 of \citep{DGL}), 
although in that case the lower bound is $d \eta/\epsilon^2$.  The bound 
for active learning is smaller by a factor of $\eta$ because the active 
learner can avoid making repeated queries to the ``heavy'' point $x_o$, 
whose label is immediately obvious.

\section{Implementing IWAL}
\label{sec:easy-alg}
IWAL({loss-weighting}) can be efficiently implemented in the
case where $H$ is the class of bounded-length linear separators 
$\{u \in \R^d: \|u\|^2 \leq B\}$ and the loss function is convex: 
$l(z,y) = \phi(yz)$ for convex $\phi$. 

Each iteration of Algorithm~\ref{alg:threshold} involves solving 
two optimization problems over a restricted hypothesis set
$$ H_t \ = \ \bigcap_{t' < t} \left\{h\in H: \textstyle{ \frac{1}{t'}\sum_{i=1}^{t'}\frac{Q_i}{p_i}l(h(x_{i}),y_{i})\leq L_{t'}^{*}+\Delta_{t'}} \right\}.$$
Replacing each $h$ by its corresponding vector $u$, this is
$$ H_t \ = \ \bigcap_{t' < t} \left\{u \in \R^d: \|u\|^2 \leq B \mbox{\ and\ } \frac{1}{t'}\sum_{i=1}^{t'}\frac{Q_i}{p_i} \phi(u \cdot (y_i x_i)) \leq L_{t'}^{*}+\Delta_{t'} \right\}.$$
an intersection of 
convex constraints.

The first optimization in Algorithm~\ref{alg:threshold} is 
$L_T^* = \min_{u \in H_T} \ \sum_{i=1}^T \frac{Q_i}{p_i} \phi(u \cdot (y_ix_i))$,
a convex program.

The second optimization is
$\max_{u,v \in H_T} \phi(y(u \cdot x)) - \phi(y(v \cdot x)), 
\ y  \in  \{+1,-1\}$
(where $u,v$ correspond to functions $f,g$).
If $\phi$ is nonincreasing
(as it is for $0$--$1$, hinge, or logistic loss), then the solution of this
problem is $\max\{\phi(A(x)) - \phi(-A(-x)), \phi(A(-x)) - \phi(-A(x))\}$, where 
$A(x)$ is the solution of a convex program:
$
A(x) \equiv \min_{u \in H_T} \ u \cdot x .
$
The two cases inside the max correspond to the choices $y = 1$ and $y = -1$.

Thus Algorithm~\ref{alg:threshold} can be efficiently implemented for nonincreasing 
convex loss functions and bounded-length linear separators.
In our experiments, we use a simpler implementation. For the first problem 
(determining $L_T^*$), we minimize over $H$ rather than $H_T$; for the second 
(determining $A(x)$), instead of defining $H_T$ by $T-1$ convex constraints, 
we simply enforce the last of these constraints (corresponding to time $T-1$).
This may lead to an overly conservative choice of $p_t$, but by 
theorem~\ref{thm:safety}, the consistency of $h_T$ is assured.
%

\subsection{Experiments}\label{sec6:experiments}
Recent consistent active learning algorithms \citep{A^2,DHM} have
suffered from computational intractability.  This section
shows that importance weighted active learning is practical.

We implemented IWAL with loss-weighting for linear separators under
logistic loss.  As outlined above, the algorithm involves two convex
optimizations as subroutines.  These were coded using log-barrier
methods (section 11.2 of \citep{BV04}).
We tried out the algorithm on the MNIST data set of handwritten
digits by picking out the 3's and 5's as two classes, and choosing 1000
exemplars of each for training and another 1000 of each for
testing. We used PCA to reduce the dimension from 784 to 25. 
The algorithm uses a generalization bound $\Delta_t$ of the form 
$\sqrt{d/t}$; since this is believed to often be loose in high dimensions,
we also tried a more optimistic bound of $1/\sqrt{t}$. In either case,
active learning achieved very similar performance (in terms of test
error or test logistic loss) to a supervised learner that saw all the
labels. The active learner asked for less than $1/3$ of the labels.

\begin{figure}[h]
\includegraphics[width=.5\textwidth]{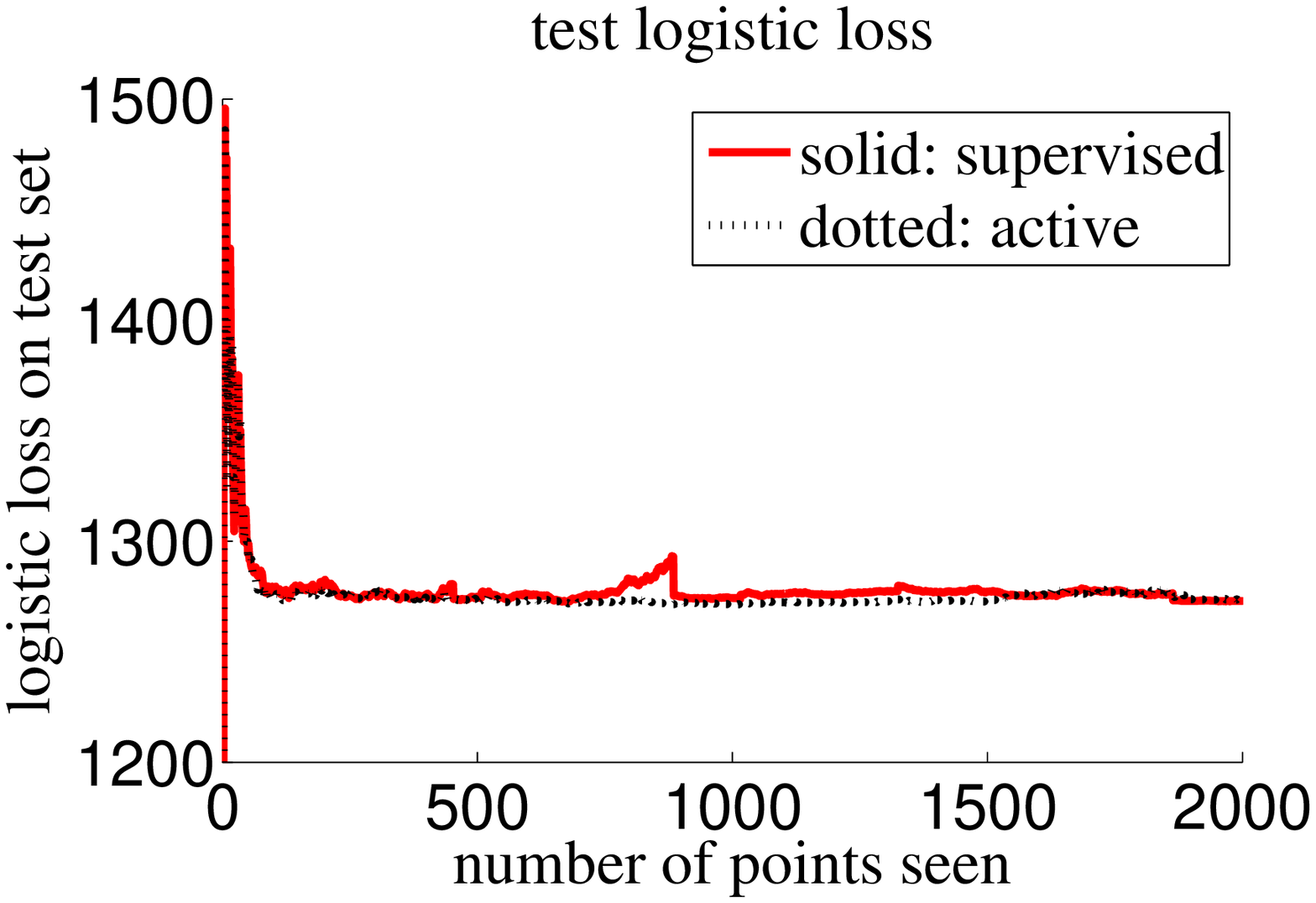}
\includegraphics[width=.5\textwidth]{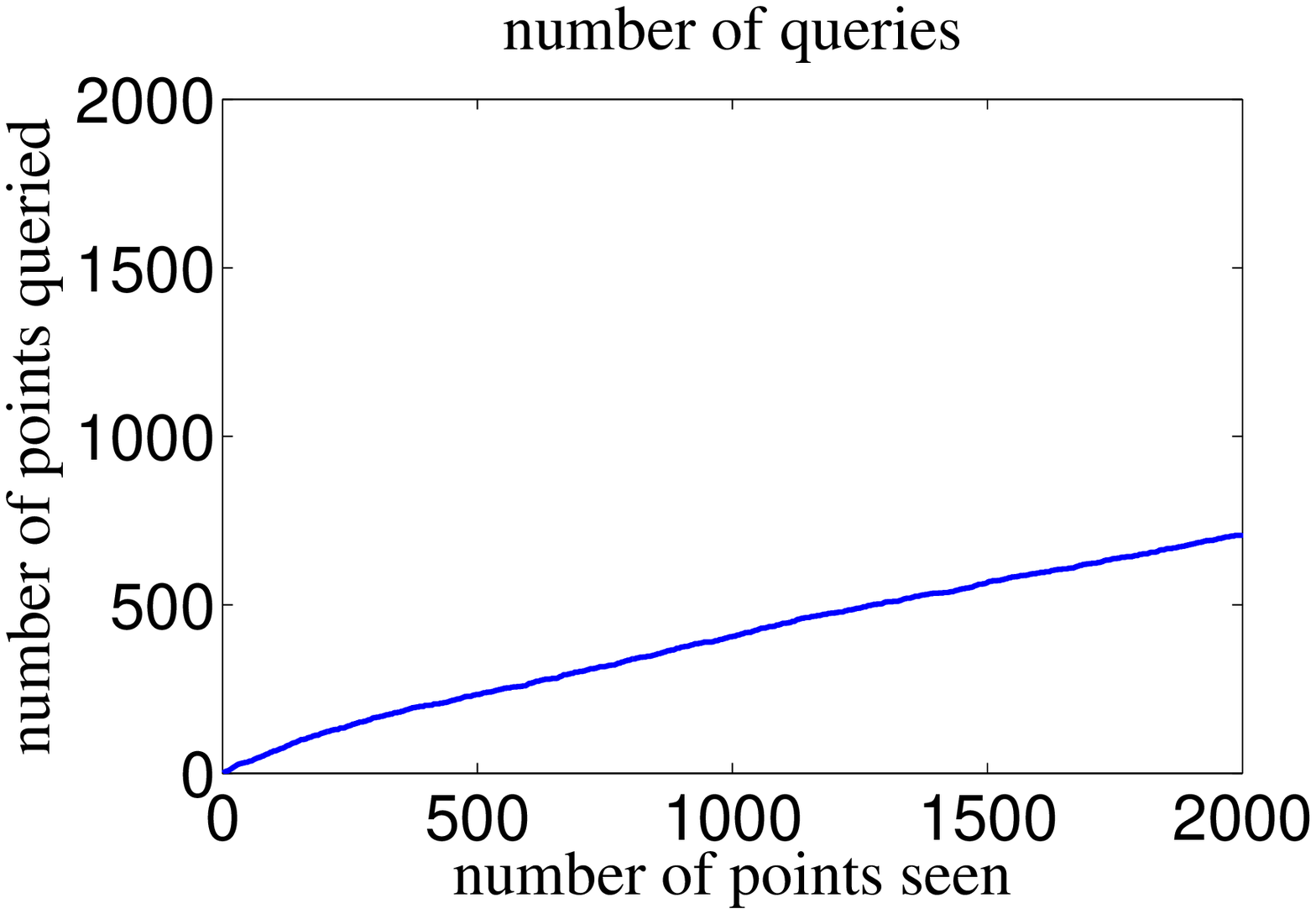}
\caption{
Top: Test logistic loss as number of points seen grows
from $0$ to $2000$ (solid: supervised; dotted: active learning).
Bottom: \#queries vs \#points seen.}
\label{fig:expts}
\end{figure}

\subsection{Bootstrap instantiation of IWAL}
\label{sec:bootstrap}
This section reports another practical implementation of IWAL, 
using a simple bootstrapping scheme
to compute the rejection threshold.  
A set $H$ of predictors is trained on some initial set of 
labeled examples and serves as an approximation of the version
space. 
Given a new unlabeled example $x$, the sampling probability is
set to $p_{\min} + (1 - p_{\min}) \big[\max_{y; h_i,h_j\in H} 
L(h_i(x),y) - L(h_j(x),y)\big]$, where 
$p_{\min}$ is a lower bound on 
the sampling probability.

We implemented this scheme for binary and multiclass classification loss,
using 10 decision trees bootstrapped on the initial 1/10th of the training set,
setting $p_{\min}=0.1$. 
For simplicity, we did't retrain the predictors for each new queried
point, i.e., the predictors were trained once on the initial sample.
The final predictor is trained on the collected
importance-weighted training set, and tested on the test set.
The Costing technique~\citep{costing} was used 
to remove the importance weights using rejection sampling.
(The same technique can be applied to any loss
function.)  The resulting unweighted classification problem was then solved
using a decision tree learner (J48).
On the same MNIST dataset as
in section~\ref{sec6:experiments}, the scheme
performed just as well as passive learning, using only
65.6\% of the labels (see Figure~\ref{fig:2}).

\begin{figure}[h]
\includegraphics[width=.28\textwidth,angle=-90]{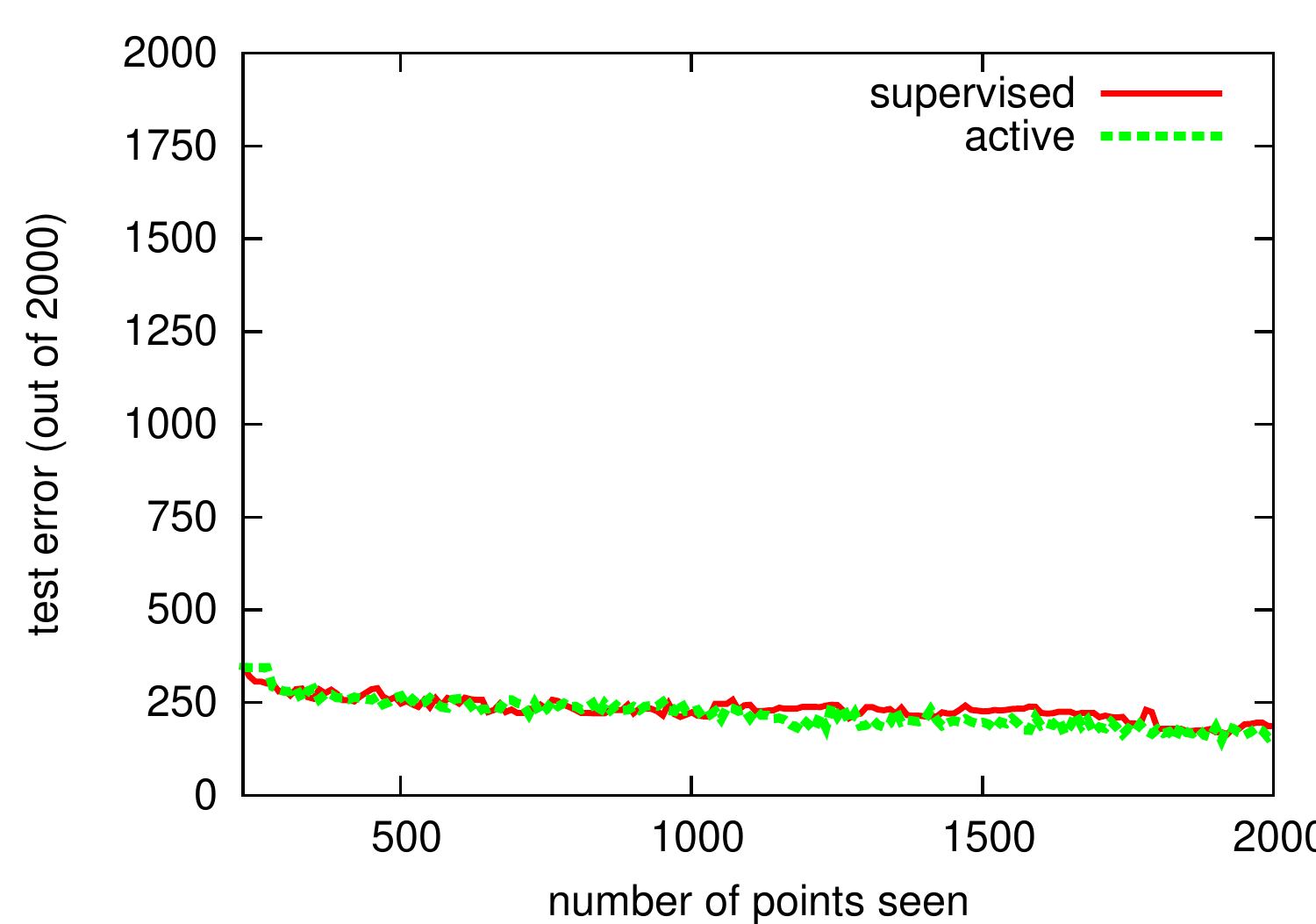}
\includegraphics[width=.28\textwidth,angle=-90]{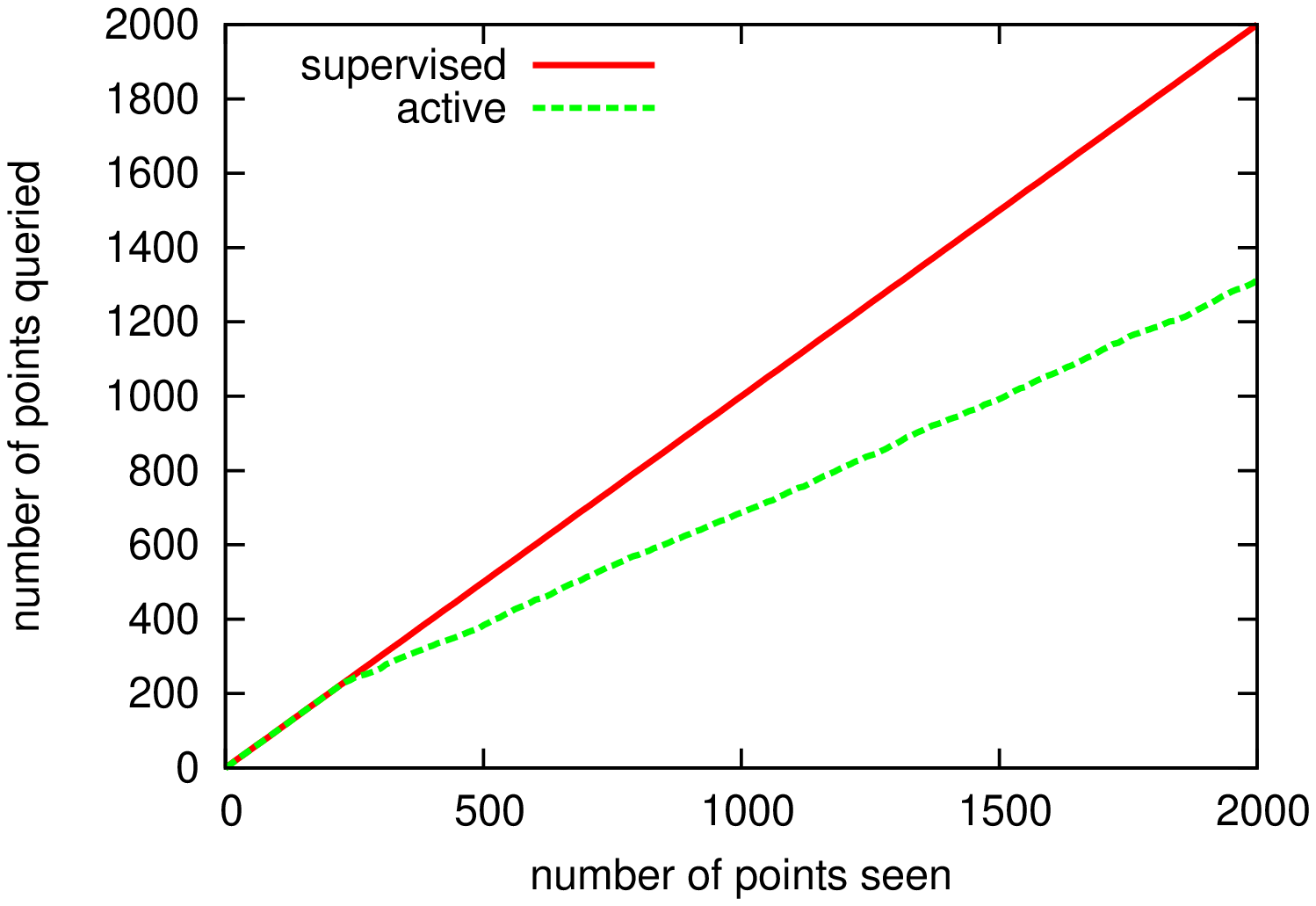}
\caption{
Top: Test error as number of points seen grows
from $200$ (the size of the initial batch, where active learning queries 
every label) to $2000$ (solid: supervised; dotted: active learning).
Bottom: \#queries vs \#points seen.}
\label{fig:2}
\end{figure}

The following table
reports additional experiments performed 
on standard benchmark datasets, 
bootstrapped on the initial 10\%. 
\begin{center}
\begin{tabular}{lcccc}
Data set & IWAL & Passive & Queried & Train/test\\
& error rate & error rate & & split \\
\hline
adult & 14.1\% & 14.5\% & 40\% & 4000/2000\\
letter & 13.8\% & 13.0\% & 75.0\% & 14000/6000\\
pima & 23.3\% & 26.4\% &  67.6\% & 538/230\\
spambase & 9.0\% & 8.9\% & 44.2\% & 3221/1380\\
yeast & 28.8\% & 28.6\% & 82.2\% & 1000/500\\
\hline
\end{tabular}
\end{center}
\ignore{
{\bf adult} (binary, roughly 4000 training and
2000 test examples, queried 40.0\%), 
{\bf pima} (binary, 538 training and
230 test examples, queried 67.6\%),
{\bf yeast} (binary, roughly 1000 training and
500 test examples, queried 82.2\%),
{\bf spambase} (binary, 3221 training and 1380 test examples, 
queried 44.2\%; bootstrapping on the initial 5\% increases query
 complexity
to 53.5\%), {\bf letter} (25 classes, 14000 training and
6000 test examples, queried 75.0\%).
For all these problems, active learning
resulted in no loss of predictive accuracy.
}

\section{Conclusion}

The IWAL algorithms and analysis presented here remove many reasonable
objections to the deployment of active learning.  IWAL satisfies the
same convergence guarantee as common supervised learning algorithms,
it can take advantage of standard algorithms (section \ref{sec:bootstrap}), 
it can deal with very
flexible losses, and in theory and practice it can yield substantial
label complexity improvements.

Empirically, in \emph{every} experiment we have tried, IWAL has
substantially reduced the label complexity compared to supervised
learning, with no sacrifice in performance on the same number of
unlabeled examples.  Since IWAL explicitly accounts for sample
selection bias, we can be sure that these experiments are valid for
use in constructing new datasets.  This implies another subtle
advantage: because the sampling bias is known, it is
possible to hypothesize and check the performance of IWAL algorithms
on datasets drawn by IWAL.  This potential for self-tuning off-policy
evaluation is extremely useful when labels are expensive.

\section{Acknowledgements}

We would like to thank Alex Strehl for a very careful reading which
caught a couple proof bugs.

\bibliography{iwal}

\end{document}